\documentclass[conference]{IEEEtran}
\IEEEoverridecommandlockouts
\usepackage{cite}
\usepackage{amsmath,amssymb,amsfonts,commath}
\usepackage{algorithm}
\usepackage{algorithmic}
\usepackage{graphicx}
\usepackage{textcomp}
\usepackage{xcolor}
\usepackage{amsthm}
\usepackage{regexpatch,fancyvrb,xparse}

\DeclareMathOperator*{\argmax}{arg\,max}
\DeclareMathOperator*{\argmin}{arg\,min}
\DeclareMathOperator{\rank}{rank}
\newtheorem{prop}{Proposition}

\def\BibTeX{{\rm B\kern-.05em{\sc i\kern-.025em b}\kern-.08em
    T\kern-.1667em\lower.7ex\hbox{E}\kern-.125emX}}
\begin{document}

\title{Neural Rule Ensembles: Encoding Sparse Feature Interactions into Neural Networks
}

\author{\IEEEauthorblockN{Gitesh Dawer}
\IEEEauthorblockA{\textit{CoreML Group} \\
\textit{Apple Inc.}\\
Cupertino, California, USA \\
dawergitesh@gmail.com}
\and
\IEEEauthorblockN{Yangzi Guo}
\IEEEauthorblockA{\textit{Department of Mathematics} \\
\textit{Florida State University}\\
Tallahassee, Florida, USA \\
yguo@math.fsu.edu}
\and
\IEEEauthorblockN{Sida Liu}
\IEEEauthorblockA{\textit{Department of Statistics} \\
\textit{Florida State University}\\
Tallahassee, Florida, USA \\
sida.liu@stat.fsu.edu}
\and
\IEEEauthorblockN{Adrian Barbu}
\IEEEauthorblockA{\textit{Department of Statistics} \\
\textit{Florida State University}\\
Tallahassee, Florida, USA \\
abarbu@stat.fsu.edu}
}

\maketitle

\begin{abstract}
Artificial Neural Networks form the basis of very powerful learning methods. 
It has been observed that a naive application of fully connected neural networks to data with many irrelevant variables often leads to overfitting. 
In an attempt to circumvent this issue, a prior knowledge pertaining to what features are relevant and their possible feature interactions can be encoded into these networks.  
In this work, we use decision trees to capture such relevant features and their interactions  and define a mapping to encode extracted relationships into a neural network. 
This addresses the initialization related concerns of fully connected neural networks. At the same time through feature selection it enables learning of  compact  representations compared to state of the art tree-based approaches. Empirical evaluations and simulation studies show the superiority of such an approach over  fully connected neural networks and tree-based approaches.
\end{abstract}

\section{Introduction}
Tree based ensemble methods have emerged as being one of the most powerful learning methods \cite{friedman2001greedy,breiman2001random} owing to the simplicity and transparency of trees, combined with an ability to explain complex data sets. 

Predictive models based on rules have gained momentum in the recent years \cite{nalenz2016horseshoe},\cite{blaszczynski2016multi},\cite{de2017best},\cite{article},\cite{deng2019interpreting}. 
One of the simplest rule based approaches was proposed in \cite{Quinlan:1993:CPM:152181} where a single decision tree is decomposed into a set of rules. 
Each such  rule is pruned by removing nodes that improved its estimated accuracy. 
This is followed by sorting the pruned rules in the ascending order of their accuracy. Prediction at any new example is obtained using a single activated rule that is highest in the sorted list.  
RuleFit \cite{friedman2008predictive} is another popular rule based predictive model. It involves generating a large pool of rules using existing fast tree growing procedures. 
The coefficients of these rules are fit through  a regularized regression. \cite{softrules} replaces the hard rules in RuleFit with soft rules using a logistic transformation. 
\cite{Dembczynski2008} employs gradient boosting using rules as a base classifiers and rules are added iteratively to an ensemble by greedily minimizing the negative log-likelihood function.  
The major concern in all of these approaches is that the activated region of rules is fixed and does not allow for any training. 
Since the support is aligned along the feature axes, a large number of rules would be required to approximate oblique decision boundaries and therefore, would result in a loose representation of the prediction function. 

Another line of work focuses on restructuring decision tree into multi-layered neural networks with sparse connections and fewer restrictions on the inclination of decision boundaries. 
One such mapping was explored in the works of \cite{sethi1991decision} which was later used by \cite{neuralforests} for every tree in a random forest.  
The mapping in \cite{sethi1991decision} replaces the Heaviside unit step function with a hyperbolic tangent activation which is known to suffer from the vanishing gradients problem. 
Also, it is not clear how to choose the hyperparameters of the hyperbolic tangent activation, which heavily dictate the initialization and the magnitude of gradients. 

Some works at the intersection of decision trees and neural networks replace every decision node with a neural network. One such study was explored by \cite{deepDecision},  who learns differentiable split functions
to guide inputs through a tree. The conditional networks from \cite{Ioannou2016DecisionFC}
also use trainable routing functions to perform conditional transformations
on the inputs which enables them to transfer computational
efficiency benefits of decision trees into the domain of convolutional networks. This appears like an ensembling of neural networks but structured in a hierarchical fashion. 

In this paper, we present a novel method called Neural Rule Ensembles (NRE) for encoding the feature interactions captured by a single decision tree into a neural network. 
We discuss some training aspects of the algorithm and perform empirical evaluations on $19$ binary classification datasets from the Penn Machine Learning Benchmark (PMLB) \cite{Olson2017PMLB}. 
To evaluate the statistical significance of the results, we use two statistical tests: Wilcoxon signed-rank test and the sign test, and individually compare NRE with Random Forests (RF), Gradient Boosted Trees (GB) and Artificial Neural Networks (ANN).

\section{Preliminaries and Notations}
We will work on regression and binary classification problems, where we are given training examples $\{(\mathbf{x_i},y_i)\in {\mathbb R}^p\times {\mathbb R},i=1,...,N\}$ and we need to find a prediction function $f_{\boldsymbol{\beta}}:{\mathbb R}^p\to {\mathbb R}$ parameterized by a vector $\boldsymbol{\beta}$, such that $f_{\boldsymbol{\beta}}(\mathbf{x_i})$ agrees with $y_i$ as much as possible. For example, for linear models, the prediction function is $f_{\boldsymbol{\beta}}(\mathbf{x})=\boldsymbol{\beta}^T\mathbf{x}$ and $\boldsymbol{\beta}\in {\mathbb R}^p$.
The agreement between $f_{\boldsymbol{\beta}}(\mathbf{x})$ and $y$ on the training examples is measured by a loss function.  
\vspace{-2mm}
\begin{equation}
L(\boldsymbol{\beta})=\sum_{i=1}^N \ell(f_{\boldsymbol{\beta}}(\mathbf{x_i}),y_i)+s(\boldsymbol{\beta})\label{eq:loss}
\vspace{-2mm}
\end{equation} that should be minimized, where $s(\boldsymbol{\beta})$ is a penalty such as the shrinkage $s(\boldsymbol{\beta})=\rho \|\boldsymbol{\beta}\|^2$ which helps with generalization. 

The loss $\ell(u,y)$ depends on the problem. For regression, it could be the square loss $\ell(u,y)=(u-y)^2$. For binary classification (when $y\in \{-1,1\}$), it could be the logistic loss $\ell(u,y)=\log(1+\exp(-uy))$, the hinge loss $\ell(u,y)=\max(1-uy,0)$, or other loss functions.

\subsection{Rule Generation}
For an input $\textbf{x}\in {\mathbb R^p}$ with real-valued continuous attributes $(x_1,x_2,...,x_p)$  , one can express a conjunctive rule in the following mathematical form: 
\begin{equation}
r(\textbf{x}) = c \prod_{j=1}^{p}  I(x_j \in S_j) \label{eq:conrule}
\end{equation}
where $I(\cdot)$ is an indicator function, $c$ is the activation value and $S_j$ is some subset of possible values for an attribute $x_j$. 

A decision tree can be regarded as a collection of conjunctive rules where each path from the root to a terminal node defines one such rule.
 Specifically, a regression tree with $m$ terminal nodes can be represented as 
\vspace{-1mm}
\begin{equation}
T(\textbf{x}) = \sum_{k=1}^{m} r_k(\textbf{x})
\vspace{-1mm}
\end{equation}
Non-zero values of $r_k(\textbf{x})$ correspond to a hyper-rectangle in the input feature space. These $m$ hyper-rectangles are non-overlapping with each other and collectively define a $m$ partitioning of the feature space.

 In cases where all the features involved in a decision tree-induced rule  are continuous variables, the rule in \eqref{eq:conrule} can be reduced to a much simpler form. Let $\textbf{x}_{R}$ be a set of features involved in a rule $r(\textbf{x})$, one can now rewrite the expression for a conjunctive rule in \eqref{eq:conrule} as follows.
\begin{equation}
r(\textbf{x}) = c  \prod_{x_j\hspace{1pt} \in \hspace{2pt}\textbf{x}_{R}} \text{H}(w_j x_j + a_j) \label{eq:rule1}
\vspace{-1mm}
\end{equation}

where $\text{H}(\cdot)$ is a Heaviside step function with $\text{H}(0)=0$,  $c$ is the value at an associated terminal node, $w_j$ is either $+1$ or $-1$ and $a_j$ is the split threshold for feature $x_j$ if $w_j=-1$, and is the negative of the split threshold otherwise.

In order to create an initial pool consisting of a large number of rules, we first invoke existing fast algorithms such as Random Forests and Gradient Boosted Trees to generate tree ensembles. 
In the subsequent step, each of the tree thus produced is decomposed into a set of conjunctive rules as described above. 

\subsection{Margin Maximizing Rules}\label{sec:MMR}

In order to perform either an implicit or explicit selection of rules, we would need a metric to quantify their importance, which can then be used to rank them in the order of their relevance. 
One such metric that can be employed is the hinge loss, also referred to as the maximum margin loss, where rules with lower values of the loss function will have higher relevance. Mathematically, the expression for this loss function is given as
\vspace{-1mm}
\begin{equation}
I\bigg(\textbf{y},\frac{ {\textbf{r}_k}}{\norm{{\textbf{r}_k}}}\bigg) = \sum_{i=1}^{N} \max \bigg(0,1 - y_i\frac{ r_k(\textbf{x}_i)}{\norm{{\textbf{r}_k}}}\bigg)\label{eq:maxMarg}
\vspace{-1mm}
\end{equation}

where $y_i \in \{-1,+1\} \hspace{1mm}\forall \hspace{1mm}i$. The quantity  $y\frac{ r_k(\textbf{x})}{\norm{{\textbf{r}_k}}}$, known as the \textit{margin} or confidence in the prediction, is positive for the correct prediction and is negative in case of the wrong one. 
Notice that,  in the absence of rule scaling  $\norm{{\textbf{r}_k}}$, the scale  of $r_k(\textbf{x})$  could be artificially chosen to make the confidence $y\hspace{1mm} r_k(\textbf{x})$ arbitrarily large which in that case would  render the definition of \textit{margin} useless. 

Let $n_k$ be the number of training examples that activate the rule $r_k$ where $n_k^{+}$ of them belong to the positive class and the remaining $n_k^{-}$ come from the negative class. Using \eqref{eq:rule1}, the euclidean norm of the rule can thus be computed as
\vspace{-1mm}
$$\norm{{\textbf{r}_k}}_2 =|c| \sqrt{n_k} 
\vspace{-2mm}
$$
For any training input $\textbf{x}$, the activation value of a rule $r_k(\textbf{x})$ is either $0$ or $c$. Consequently, the normalized value is either $0$ or is given by 
\vspace{-1mm}
$$ \frac{ r_k(\textbf{x})}{\norm{{\textbf{r}_k}}}  = \frac{c}{|c|}\frac{1}{\sqrt{n_k}} 
\vspace{-1mm}
$$
which implies that the \textit{margin} is unconditionally bounded from above by $1$ or more precisely, 
\vspace{-1mm}
$$ -1 \leq y\frac{ r_k(\textbf{x})}{\norm{{\textbf{r}_k}}} \leq 1, \text{ so }
 0 \leq 1 - y\frac{ r_k(\textbf{x})}{\norm{{\textbf{r}_k}}} \leq 2.
 \vspace{-1mm}
$$
This allows us to simplify the expression for the hinge loss given in \eqref{eq:maxMarg} to
\vspace{-1mm}
\begin{align*}
 I\bigg(\textbf{y},\frac{ {\textbf{r}_k}}{\norm{{\textbf{r}_k}}}\bigg) & = \sum_{i=1}^{N} \Bigg(1 - y_i\frac{ r_k(\textbf{x}_i)}{\norm{{\textbf{r}_k}}}\Bigg) \\\vspace{1mm}
 & =  N - \sum_{i=1}^{N} y_i\frac{ r_k(\textbf{x}_i)}{\norm{{\textbf{r}_k}}} 
 =  N -  \hspace{-4mm}\sum_{ \{ i :r_k(\textbf{x}_i) \neq 0 \}} y_i \frac{ r_k(\textbf{x}_i)}{\norm{{\textbf{r}_k}}}
\vspace{-1mm}
\end{align*}
Since both $r_k(\textbf{x})$ and $-r_k(\textbf{x})$ are potentially valid rules, the quantifying metric becomes
\vspace{-1mm}
\begin{equation}
\argmin_{k}  I\bigg(\textbf{y},\pm \frac{ {\textbf{r}_k}}{\norm{{\textbf{r}_k}}}\bigg)  = 
\argmax_k \hspace{1mm}m_k^2 \label{eq:Ikmk}
\vspace{-1mm}
\end{equation}
where
\vspace{-1mm}
\begin{equation}
 m_k = \sum_{ \{ i :r_k(\textbf{x}_i) \neq 0 \}} y_i \frac{ r_k(\textbf{x}_i)}{\norm{{\textbf{r}_k}}} \label{eq:mkhingeloss}
\vspace{-1mm}
\end{equation}
This means that the rules can simply be sorted in the decreasing order of  the scores $m_k^2$ to obtain a ranking in the order of their relevance. 
In other words, higher values of $m_k^2$ correspond to higher  relevance. Let us further simply the expression given in \eqref{eq:mkhingeloss} to 
\begin{equation}
m_k  = \frac{c}{|c|}\frac{(n_k^{+} - n_k^{-})}{\sqrt{n_k}}, \text { thus }
m_k^2  = \frac{(n_k^{+} - n_k^{-})^2}{n_k}. \label{eq:mksquare} 
\end{equation}

\section{Rule Generation Limitation of Conventional Decision Tree} \label{sec:limitationRSA}

The rule generation procedure conducted by conventional decision tree uses only one feature per node for recursive binary partitioning which corresponds to the fact that all the slices or partitions in the feature space are now either perpendicular or (inclusive) parallel to the feature axes. 
In other words, such models will have difficulties in approximating oblique decision boundaries in the feature space. 

Let us consider a linearly separable dataset as shown in Figure \ref{fig:limitRER}(a), where the decision boundary is inclined at an angle of $45^{\circ}$ with either of the feature axes. 
Owing to the aforementioned rigidity, it can be seen from Figures \ref{fig:limitRER} (b) and (c) that  the only way to improve performance is to keep adding more rules to the ensemble. The inability to evolve shapes of the rules seems highly restrictive and challenges the goal of achieving compact representations. Hence, we would not want to restrict the activated region of a rule $r(\textbf{x})$ to just a hyper-rectangle. This motivates the definition of a neural rule with  trainable support. 

\begin{figure}[h!]
	\vspace{-3mm}
	\centering
	\begin{center}	
	\begin{tabular}{ccc}	
	\includegraphics[width=0.26\linewidth]{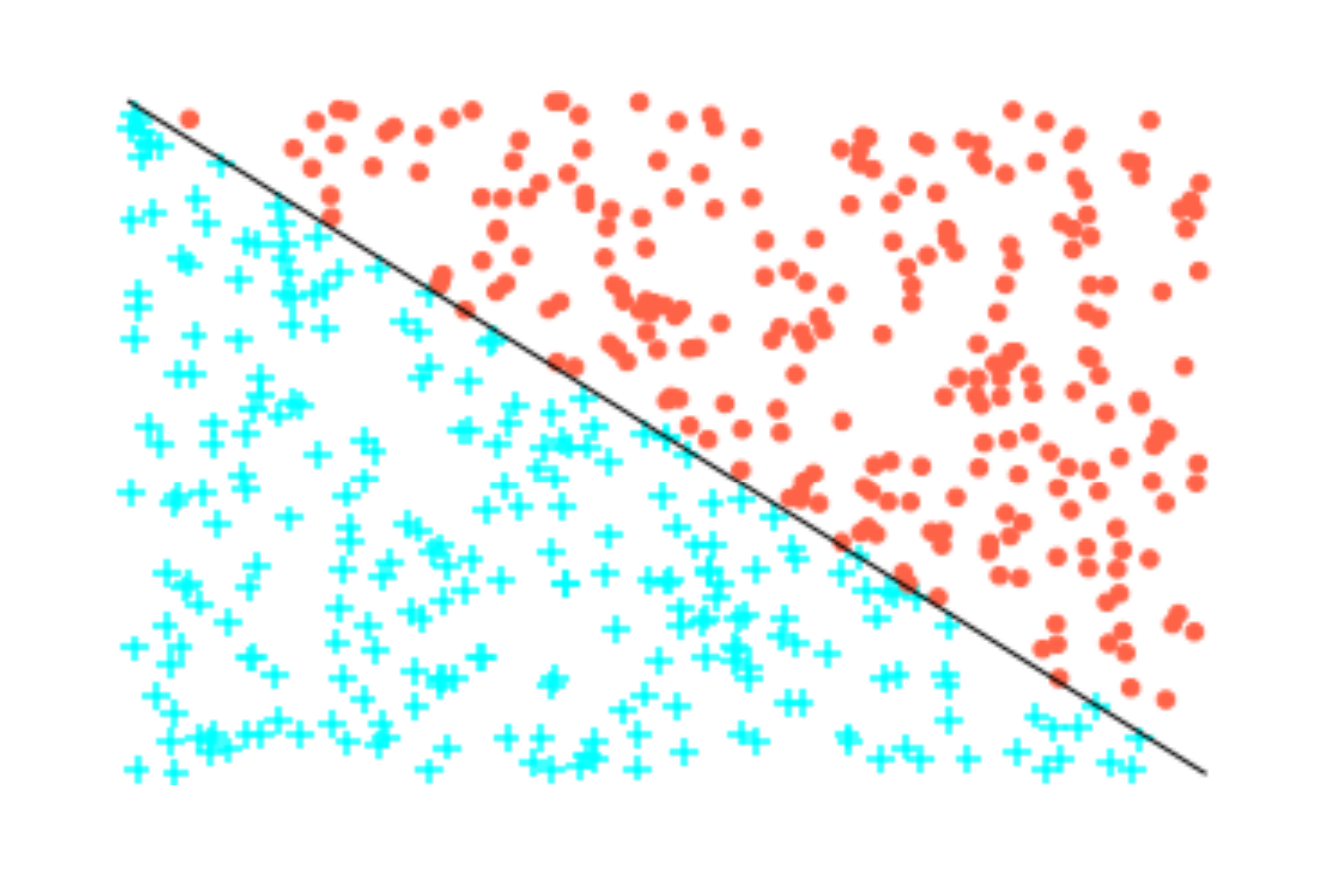}
	&\includegraphics[width=0.26\linewidth]{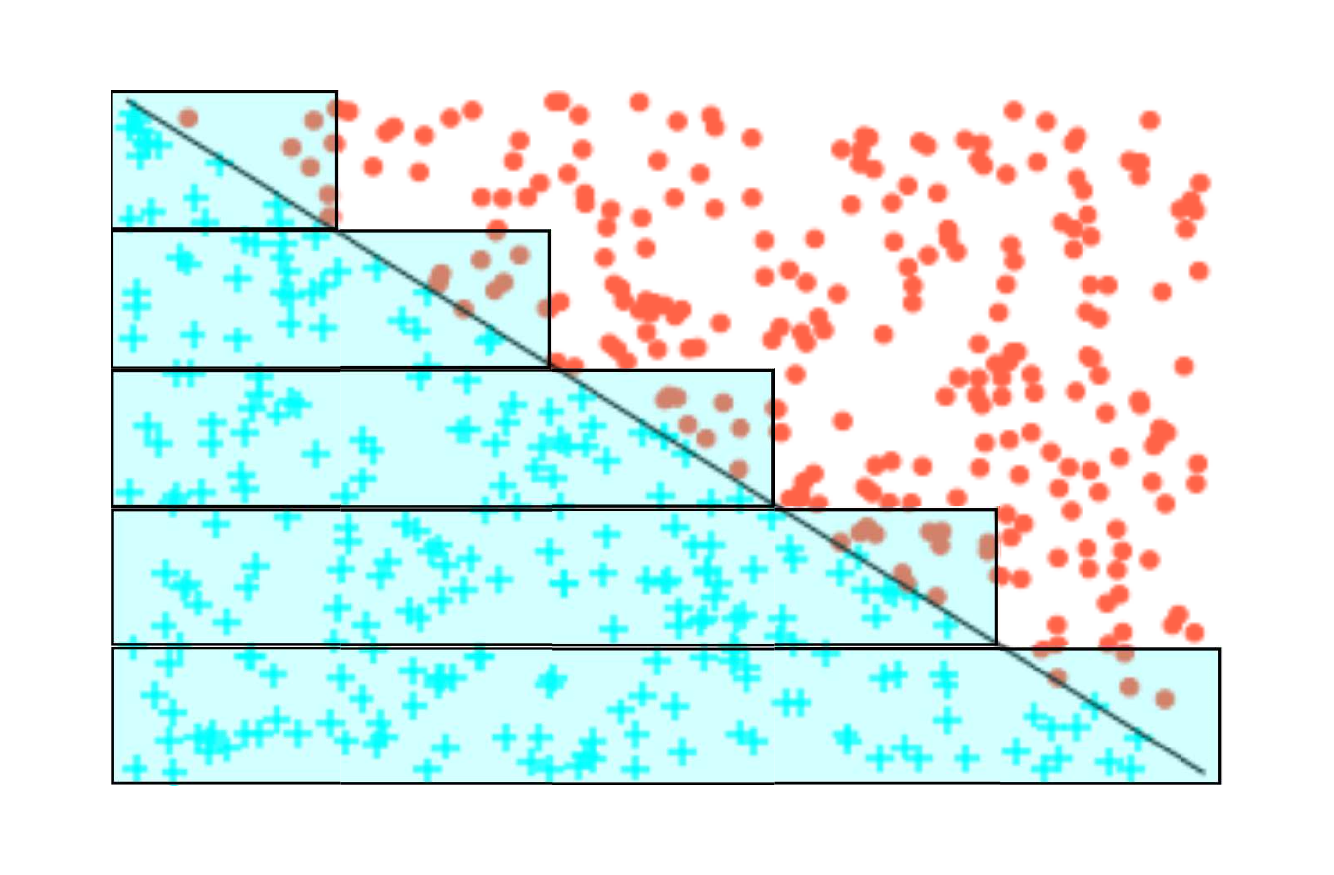}
	&\includegraphics[width=0.26\linewidth]{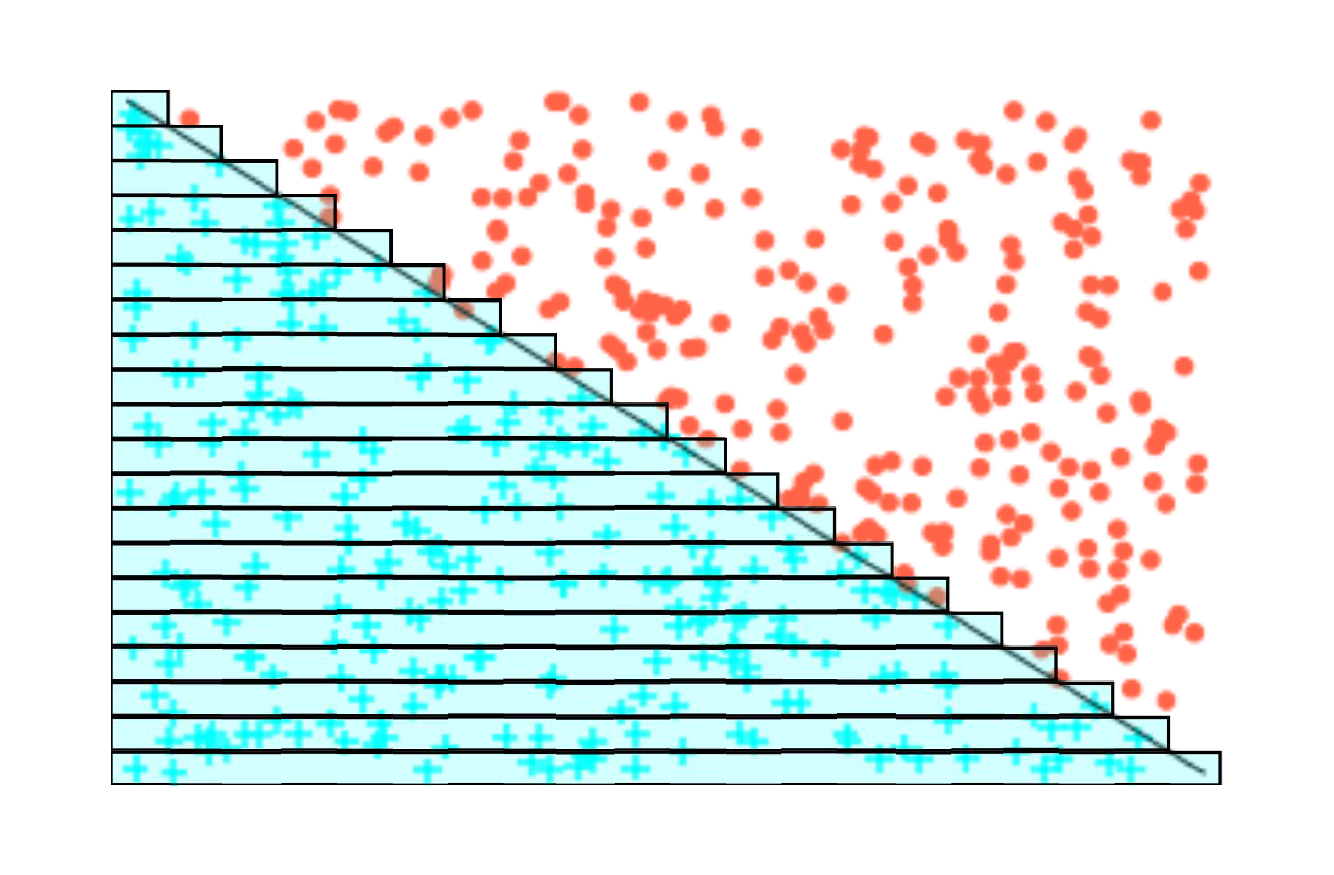}\\
	Linear Separable & Ensemble 5 Rules & Ensemble 20 Rules
	\end{tabular}
	\end{center}
	\vspace{-4mm}
	\caption{Limitations of Conventional Ensemble of Rules in approximating linearly separable datasets}\label{fig:limitRER}
	\vspace{-1mm}
\end{figure}

\section{Neural Rule Ensembles} \label{sec:NRInit}

In this part, we introduce a new form of the conjunctive rules  \eqref{eq:rule1} called Neural Rules. 
In order to see how a decision tree based rule inspires the design of a neural rule, let us revisit the expression of a conjunctive rule specifically for a decision tree based rule given in \eqref{eq:rule1}. 
\begin{equation}
r(\textbf{x}) = c  \prod_{x_j\hspace{1pt} \in \hspace{2pt}\textbf{x}_{R}} \text{H}(w_j x_j + a_j) \nonumber
\end{equation}
where $\textbf{x}_{R}$ is a set of features involved in a rule $r(\textbf{x})$, $\text{H}(\cdot)$ is a Heaviside step function with $\text{H}(0)=0$,  $c$ is the value at an associated terminal node, $w_j$ is either $+1$ or $-1$ and $a_j$ is the split threshold for feature $x_j$ if $w_j=-1$, and is the negative of split threshold otherwise.

Rules extracted from a decision tree involve only one feature for every node. In a neural rule, we modify \eqref{eq:rule1} to now connect each node of a given rule with all the features used in the corresponding decision tree. 
Let $\textbf{x}_{T}$ be a vector of all features  used in the decision tree $T(\textbf{x})$ without repetitions. 
With this modification, we can now have oblique decision boundaries in the feature subspace spanned by $\textbf{x}_{T}$. The updated expression of the rule looks as follows.

\vspace{-2mm}
\begin{equation}
r(\textbf{x}) = c  \prod_{\{j: \hspace{1pt}x_j \in \hspace{1pt}\textbf{x}_{R} \}}\text{H}(\textbf{w}_j^T {\textbf{x}}_T + a_j)  \label{eq:rule2}
\vspace{-2mm}
\end{equation}
Denote the ReLU operation as $\sigma(x) = \max(0,x)$. Next, we observe that a Heaviside step function with $\text{H}(0)=0$ is invariant to the ReLU transformation of an input $\text{H}(x)  = \text{H}(\sigma(x))$. 
Also note that the product of several Heaviside step functions can be represented using a single Heaviside step function and the minimum pooling operation.
\vspace{-1mm}
\[
\prod_k \text{H}(x_k)  =  \text{H}(\underset{k}{\min} \hspace{3pt}x_k) 
\vspace{-1mm}
\]
Using these identities, the rule in \eqref{eq:rule2} becomes,
\vspace{-1mm}
\begin{align}
r(\textbf{x}) &= c \prod_{\{j: \hspace{1pt}x_j \in \hspace{1pt}\textbf{x}_{R} \}}\text{H}(\sigma(\textbf{w}_j^T {\textbf{x}}_T + a_j))\\
& = c \cdot \text{H}(\underset{\{j: \hspace{1pt}x_j \in \hspace{1pt}\textbf{x}_{R} \}}{\min} \sigma(\textbf{w}_j^T {\textbf{x}}_T + a_j)) \label{eq:rule3}
\vspace{-2mm}
\end{align}

Since the derivative of a step function is zero, the gradients of all the learnable parameters will stay zero unless some modification is made. In order to be able to jointly train all the weights and splitting thresholds of all the nodes in a rule, we switch the Heaviside step function in previous equations with an identity function. This gives us the neural rule in its final form as
\vspace{-2mm}
\begin{equation}
r({\textbf{x}}) = c \underset{\{j: \hspace{1pt}x_j \in \hspace{1pt}\textbf{x}_{R} \}}{\min} \sigma(\textbf{w}_j^T {\textbf{x}}_T + a_j)\label{eq:neuralrule}
\vspace{-1mm}
\end{equation}

\subsection{Initialization} \label{sec:nreinit}

We now discuss an initialization of a neural rule corresponding to any given rule obtained from a decision tree. 

First, we  make a list of all the  features involved in that tree. All those features along with a bias unit are the input layer of a neural rule. 
The number of hidden units in the first hidden layer of a neural rule equals the number of decision nodes of a tree-induced rule, with one-to-one correspondence between them. 
The connection weight between the input feature and the hidden unit is $0$ if the corresponding decision node of that hidden unit does not involve the feature under consideration. 
It is $-1$ if the corresponding decision node utilizes that feature and traverses its left child along the rule path and is $+1$ otherwise. 

The magnitude of the bias for every hidden unit is given by the absolute value of the splitting threshold utilized in the corresponding decision node. 
The sign of the bias for a hidden unit is positive if the corresponding decision node traverses its left child along the rule path, otherwise it is negative. 

We use Figure \ref{fig:map} to illustrate one such mapping. 
Figure \ref{fig:map} (a) shows a decision tree with four rules. 
Figure \ref{fig:map} (b) shows a neural rule corresponding to the rule with terminal label $2$ (red colored branch) of the decision tree.  
All the bold lines in a neural rule represent trainable parameters with their initial values displayed alongside in Figure \ref{fig:map} (b). 
The red bold lines in a neural rule carry non-zero initial weights and have their counterparts in the decision tree whereas black bold lines represent new connections with zero initialization.

\begin{figure}[t]
\begin{center}
\begin{tabular}{cc}
\includegraphics[width=0.2\textwidth]{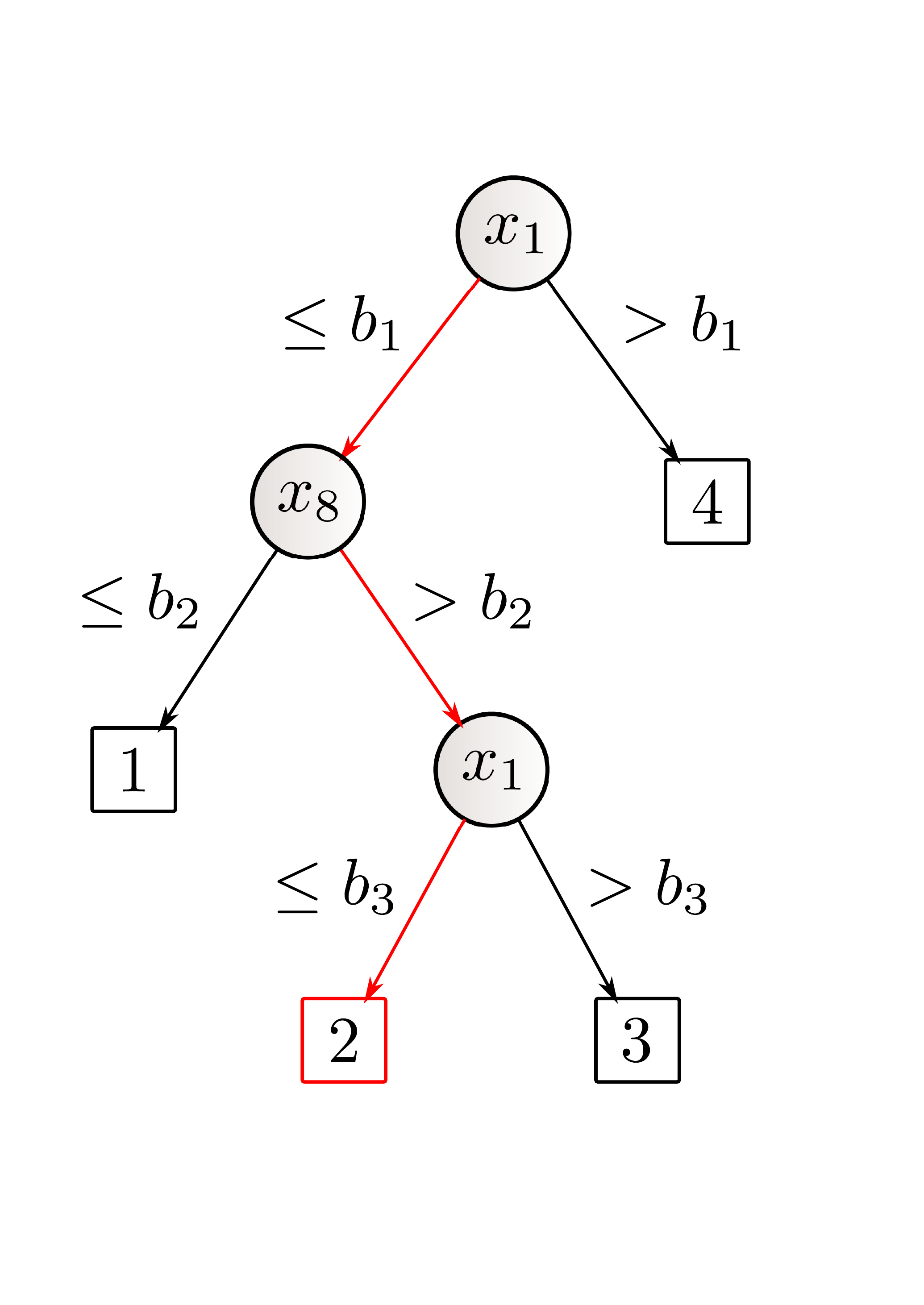}
&\includegraphics[width=0.2\textwidth]{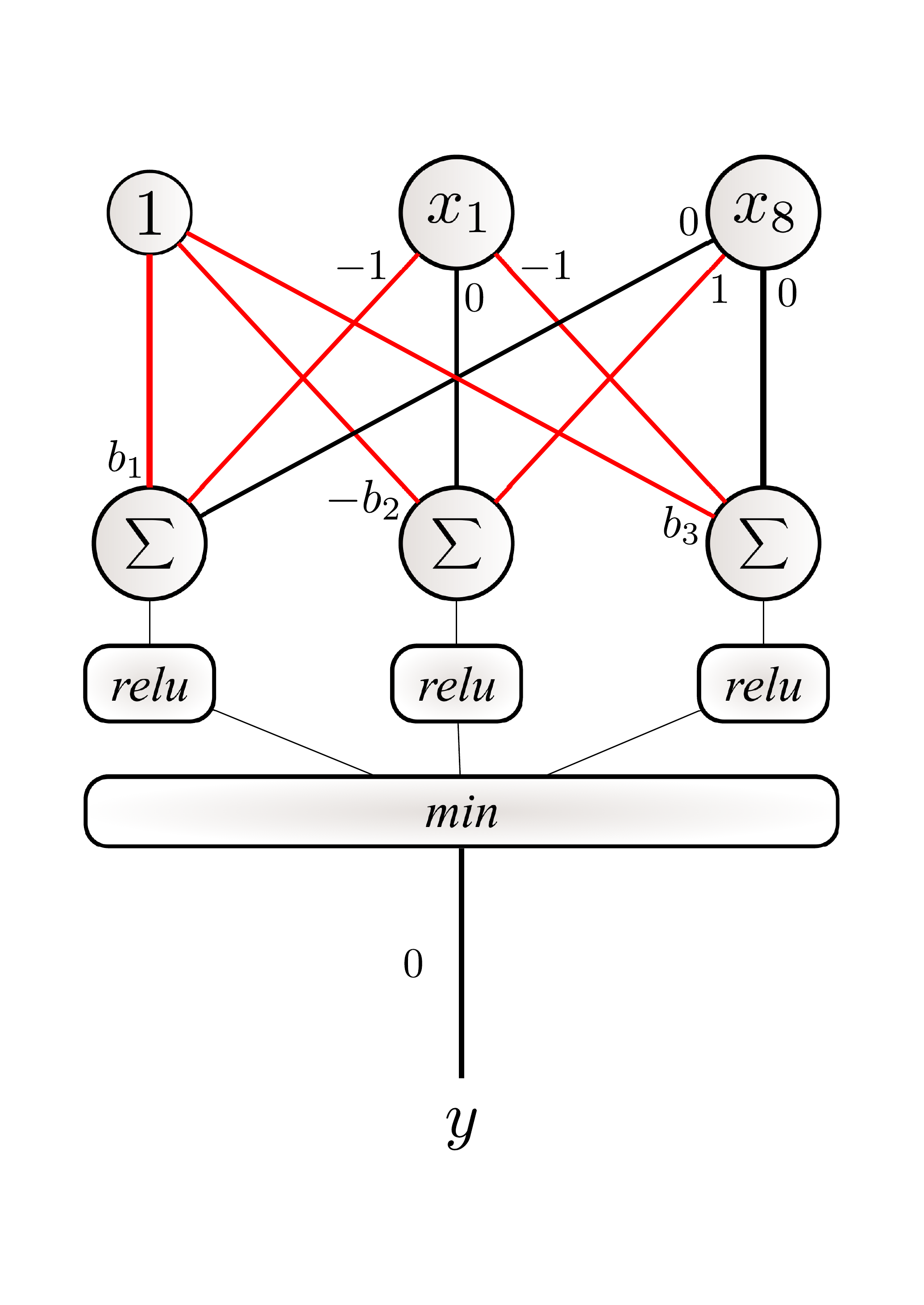}\vspace{-6mm}\\
A Decision Tree & A Neural Rule
\end{tabular}
\end{center}
	\caption{Mapping a  Tree-induced Rule, with terminal label 2, into a Neural Rule}\label{fig:map}
\end{figure}

We observed that the support of a proposed neural rule is a convex set. 
In order to allow for the rules to assume complicated non-convex shapes in the feature space, we extend the definition of a neural rule by stacking a new hidden layer with the same number of hidden units as the previous one. 
We refer to such a modification of the neural rule as a deep neural rule. 
Since we need to preserve the support of a tree-induced rule while mapping it into a corresponding deep neural rule, we  use an identity transformation for initializing the parameters of this new hidden layer as depicted in  \figref{fig:deepneuralrule}.

\subsection{Characteristics}
\noindent{\bf Trainable Support}.
Each $\textbf{w}^T {\textbf{x}}_T + a = 0$ in the equation \eqref{eq:neuralrule} represents a hyperplane in the feature subspace with the corresponding upper half space given by $ \sigma(\textbf{w}^T {\textbf{x}}_T + a) > 0$. 
The application of the $\min$ operation evaluates  the intersection of these upper half spaces and thus, defines the activated region of a rule. 
For an input $\textbf{x}$ that lies on the upper half space of the plane given by $\textbf{w}^T {\textbf{x}}_T + a = 0$, $\sigma(\textbf{w}^T {\textbf{x}}_T + a)$ is proportional to the shortest euclidean distance of the input to that hyperplane. 
This quantifies the  margin or level of confidence in the prediction and the further the input lies from the hyperplane in its upper half space, the more confident it becomes in its prediction of that input. 
During training using backpropagation, the hyperplane is rotated, shifted and scaled in order to maximize the expected margin of the inputs. 
Because of the $\min$ pooling operation, each input contributes in updating the parameters of only the hyperplane that predicts the least margin for it at that training step among all the hyperplanes  involved in a rule. 
The rationale here is to maximize the margin of an input only from the least confident hyperplane.

%

\begin{figure}[t]
	\begin{center}	
	\begin{tabular}{cc}
	\includegraphics[width=0.2\textwidth]{decisiontree.pdf}
	&\includegraphics[width=0.2\textwidth]{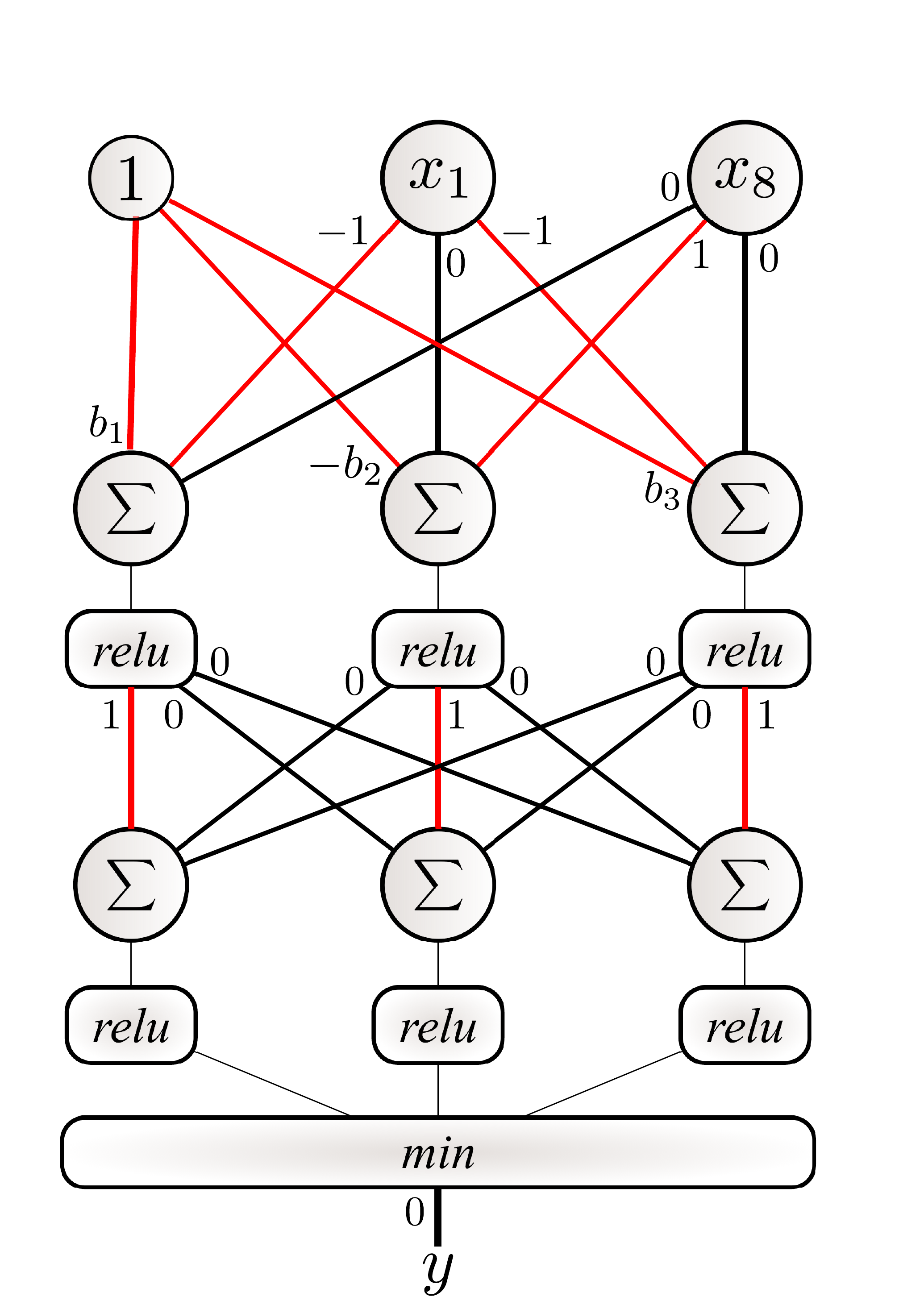}\vspace{-1mm}\\
	A Decision Tree & A Deep Neural Rule
	\end{tabular}
	\end{center}
	\caption{Mapping a Tree-induced Rule, with terminal label 2, into a Deep Neural Rule}\label{fig:deepneuralrule}
\end{figure}

\noindent{\bf Restricted Gradients}.
For an input $\textbf{x}$ that does not belong to the activated region of a rule, $ \underset{j}{\min}\hspace{2pt}( \textbf{w}_j^T {\textbf{x}}_T + a_j)$ would be less than or equal to zero, which implies zero gradients of all the trainable parameters as the derivative of a ReLU activation for negative inputs is zero. 
This suggests that only the training examples lying inside the activated region are responsible for modifying the shape of this region. 
In order to maximize their margins, activated examples try to push or pull the rule boundaries depending on the sign of their class membership and the sign of weight, $c$. 
If both of these signs agree then the corresponding training examples push  the rule boundary outwards, which expands the region and as a result, brings in more training examples. 
Otherwise, if the signs do not match then those contradictory examples will pull in the rule boundary to get themselves out of it and thereby, shrink the region.

\noindent{\bf Compact Convex Support}.\label{sec:convexsupport}
Let $C$ denote the support of a neural rule given by equation \eqref{eq:neuralrule} and defined as  ${\{ \textbf{x}  :r(\textbf{x}) \neq 0 \}}$. We show that $C$ is a compactly supported convex set. 
\begin{prop}\label{prop:convexSupport}
	For  any $\textbf{x},\textbf{z} \in C$, the convex combination of $\textbf{x}$ and $\textbf{z} $ satisfies
	$$\theta \textbf{x} + (1-\theta) \textbf{z} \in C$$
	where $\theta \in {\mathbb R}$ with $0\leq \theta \leq 1$
\end{prop}
\begin{proof} Given any $\textbf{x} \in C$, we have from equation \eqref{eq:neuralrule}, 
	$$r(\textbf{x}) \neq 0 \iff \textbf{w}_j^T {\textbf{x}} + a_j > 0\hspace{2mm} \forall j$$ 
	By multiplying both sides by $\theta > 0$, we get
	\begin{equation}
	\textbf{w}_j^T \theta{\textbf{x}} + \theta a_j > 0\hspace{2mm} \forall j \label{eq:convex1}
	\end{equation}
	Similarly, we have for any $\textbf{z} \in C$,
	$$r(\textbf{z}) \neq 0 \iff \textbf{w}_j^T {\textbf{z}} + a_j > 0 \hspace{2mm}  \forall j$$ 
	Multiplying both sides with $(1-\theta) > 0$, 
	\begin{equation}
	\textbf{w}_j^T (1-\theta){\textbf{z}} + (1-\theta)a_j >0 \hspace{2mm}  \forall j \label{eq:convex2}
	\end{equation}
	Adding equations \eqref{eq:convex1} and \eqref{eq:convex2}, we obtain
	$$ \textbf{w}_j^T (\theta{\textbf{x}} + (1-\theta){\textbf{z}}) + a_j > 0 \hspace{2mm} \forall j$$
	which implies 
	$$ r(\theta{\textbf{x}} + (1-\theta){\textbf{z}}) \neq 0 \iff \theta{\textbf{x}} + (1-\theta){\textbf{z}} \in C.$$
\end{proof}

\subsection{Training}
Conventional procedures for generating decision tress on binary classification tasks employ either the Gini index or the cross entropy measure. 
We use a new splitting criterion for invoking a decision tree based on margin maximizing rules discussed in \secref{sec:MMR}. 
Let $n$ denote the number of examples. We use  subscript $l$ to refer to the left child, $r$ for the right child and $p$ for the parent node. 
Additionally, superscripts $+$ and $-$  refer to positive and negative examples, respectively.

Assuming binary partitioning of a decision node,  each split defines two simple rules $r_l(\textbf{x})$ and $r_r(\textbf{x})$. 
Using the  maximum margin metric for a rule given in \eqref{eq:mksquare}, the node splitting criterion can be written as follows
\begin{equation}\label{eq:treesplit}
I = \frac{(n_l^{+} - n_l^{-})^2}{n_l} + \frac{(n_r^{+} - n_r^{-})^2}{n_r} -  \frac{(n_p^{+} - n_p^{-})^2}{n_p} 
\vspace{-1mm}
\end{equation}

We decompose a single decision tree  into a set of conjunctive rules to obtain a pool of diverse feature interactions. 
Each of these rules is used to initialize their neural counterparts using mapping discussed in Section \ref{sec:nreinit}. 
Such an ensemble of neural rules, collectively referred to as Neural Rule Ensembles (NRE) is essentially a 2-layered artificial neural network with min pooling operation and thus, a universal approximator \cite{hornik1991approximation}. 
However, in the proposed approach, feature interactions extracted from a decision tree are  explicitly encoded into the network through its initialization, thus performing feature selection and leading to better generalization. 
Another characteristic of such an initialization is that the activations of any two pooled hidden units are orthogonal to each other.  

After initializing the network, all the parameters are trained using Backpropagation \cite{backprop}. 
We use the Adam optimization method \cite{adamOptimization} with learning rate $\alpha = 0.01$ to calculate the weight updates. 

\begin{algorithm}[htb]
\caption{ Neural Rule Ensemble (NRE) Training}
\label{alg:algorithm}
\textbf{Input}: Training data $\{(\mathbf{x_i},y_i)\}_{i=1}^N$\\
\textbf{Parameters}: Learning rate $\alpha = 0.01$, number of training epochs $N^{ep}$\\
\textbf{Output}: Trained Neural Rules Ensemble

\begin{algorithmic}[1] 
		\STATE Standardize the training data $\{(\mathbf{x_i},y_i)\}_{i=1}^N$: mean centering and unit standard deviation 
		\STATE Build a decision tree using the splitting criterion from Eq. \eqref{eq:treesplit}
		\STATE Decompose the resulting tree into a set of conjunctive rules
		\STATE Map each tree-induced rule into a corresponding neural rule 
		\STATE Initialize each neural rule as detailed in Section (IV-$A$).
		\STATE  Train an ensemble of neural rules simultaneously using backpropagation 
\STATE \textbf{return} Ensemble of Neural Rules
\end{algorithmic}
\end{algorithm}

\section{Experiments}

\subsection{Simulation Result}

In this section, we perform a simulation to illustrate the ability of a neural rule to evolve and expand its activated region. 
We consider a rotated XOR dataset, which is a non-linearly separable dataset since there does not exist any single hyperplane that can separate the positive training examples (shown in blue) from the negative ones (shown in red). 
Additionally, since we have rotated the XOR dataset  by $45^{\circ}$,  tree-based approaches such as Gradient boosted trees would have a hard time approximating the oblique decision boundaries and would require a large number of trees and/or rules. 

\figref{fig:evolutionNeuralRule} (a) shows a single neural rule just after its initialization from a corresponding tree-induced rule. \figref{fig:evolutionNeuralRule} (b) shows an intermediary state after training for 150 iterations. 
It can be seen that the rule evolves its activated region to include more examples of the same type into its support. 
After training for a long time, the neural rule settles into an equilibrium state consisting of only positive examples as shown in \figref{fig:evolutionNeuralRule} (c).  
\begin{figure}[t]
	\begin{center}	
	\begin{tabular}{ccc}	
	\includegraphics[width=0.26\linewidth]{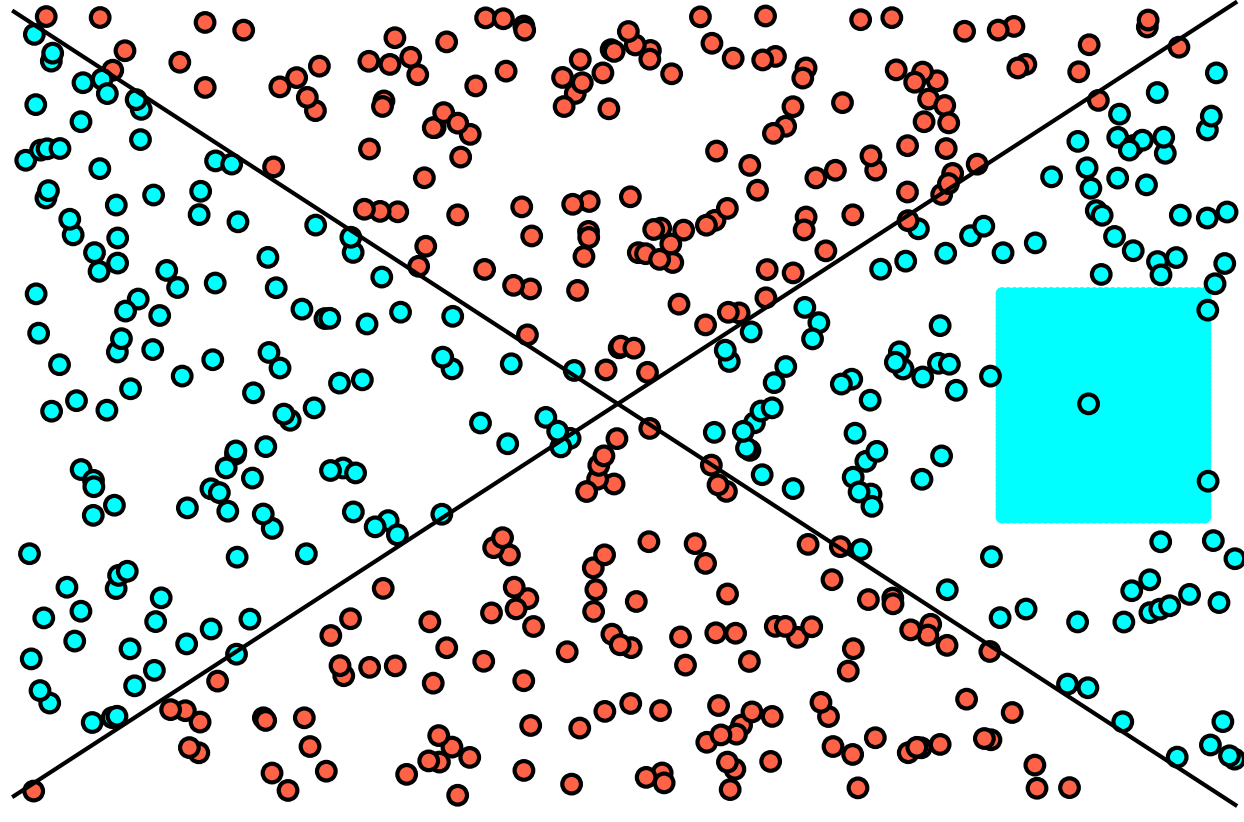}
	&\includegraphics[width=0.26\linewidth]{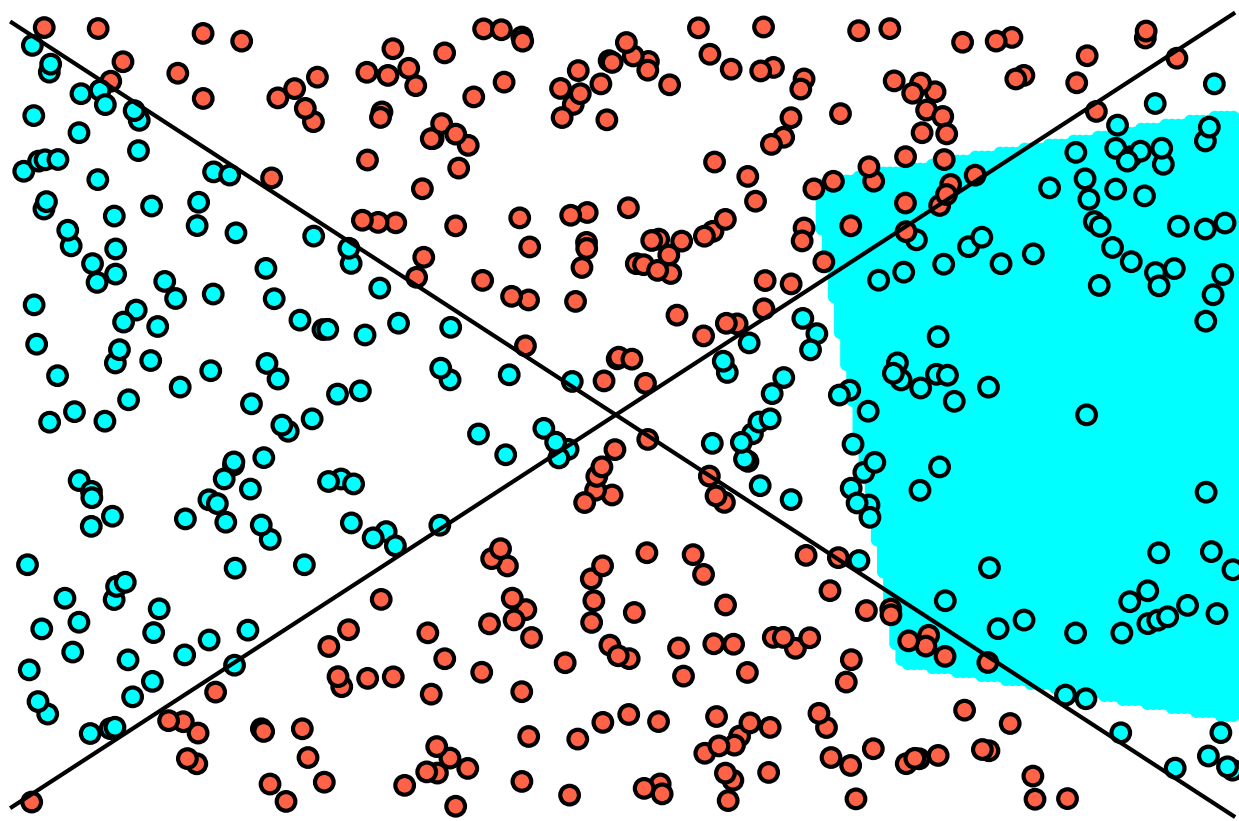}
	&\includegraphics[width=0.26\linewidth]{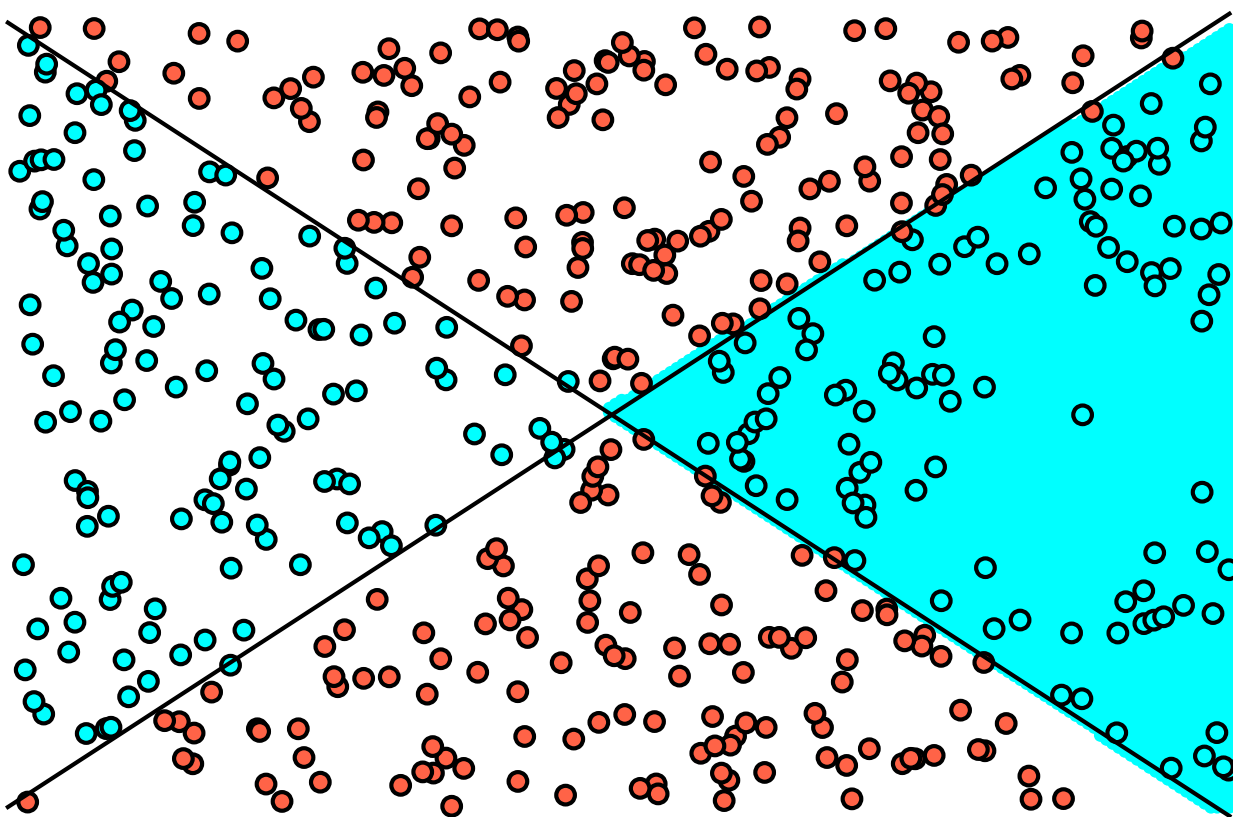}\\
	Initialization & After 150 Iterations & After 3k Iterations
	\end{tabular}
	\end{center}
	\vspace{-4mm}
	\caption{Neural Rule: Evolution of the trainable support of a single rule with time}\label{fig:evolutionNeuralRule}
	\vspace{-2mm}
\end{figure}

It is evident from \figref{fig:evolutionNeuralRule} (c) that there are still many positive training examples that do not belong to the support of a neural rule. 
In order to include them, a neural rule would have to assume a non-convex shape, which is not possible. 
This limitation has motivated the extension to a deep neural rule. 
\figref{fig:evolutionDeepNeural} shows the evolution of the activated region in the case of a deep neural rule, which can now achieve a non-convex shape and hence contain all the positive training examples into its support as shown in \figref{fig:evolutionDeepNeural} (c).

\begin{figure}[h!]
	\vspace{-3mm}
	\centering
	\begin{center}	
	\begin{tabular}{ccc}	
	\includegraphics[width=0.26\linewidth]{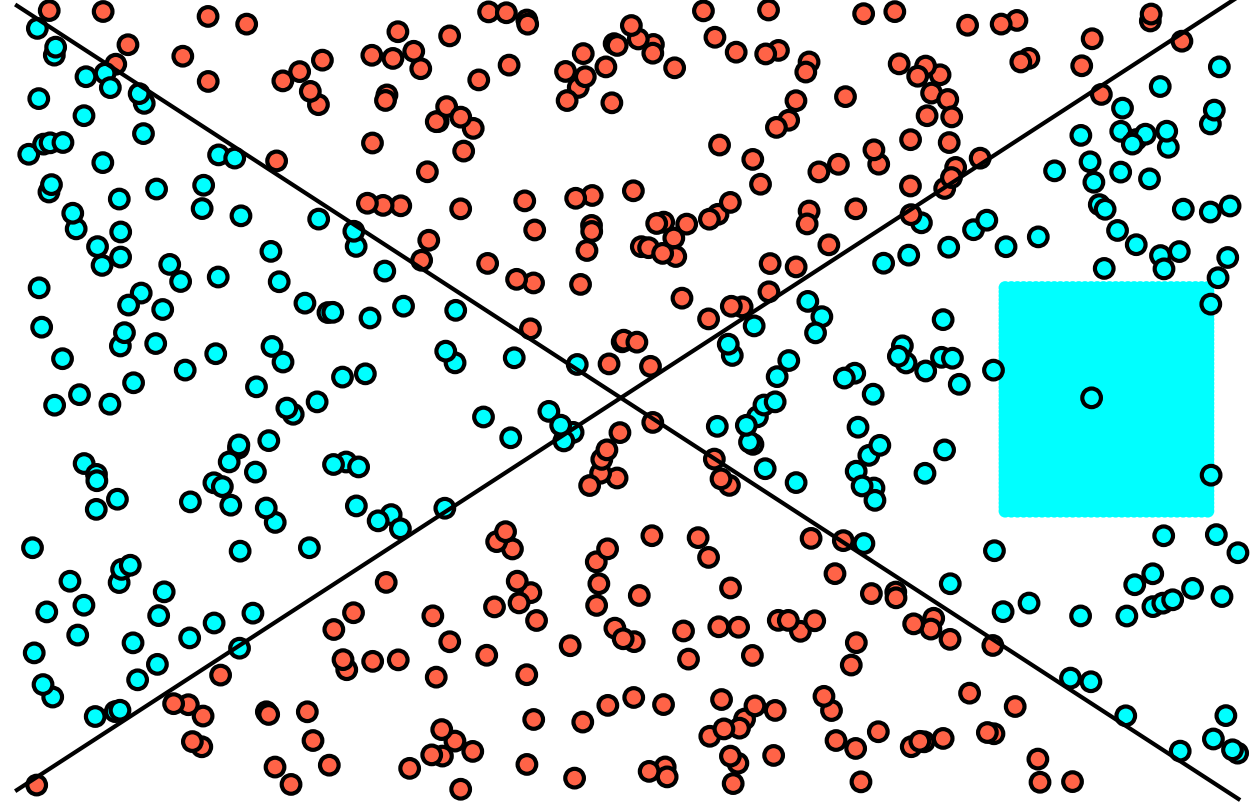}
	&\includegraphics[width=0.26\linewidth]{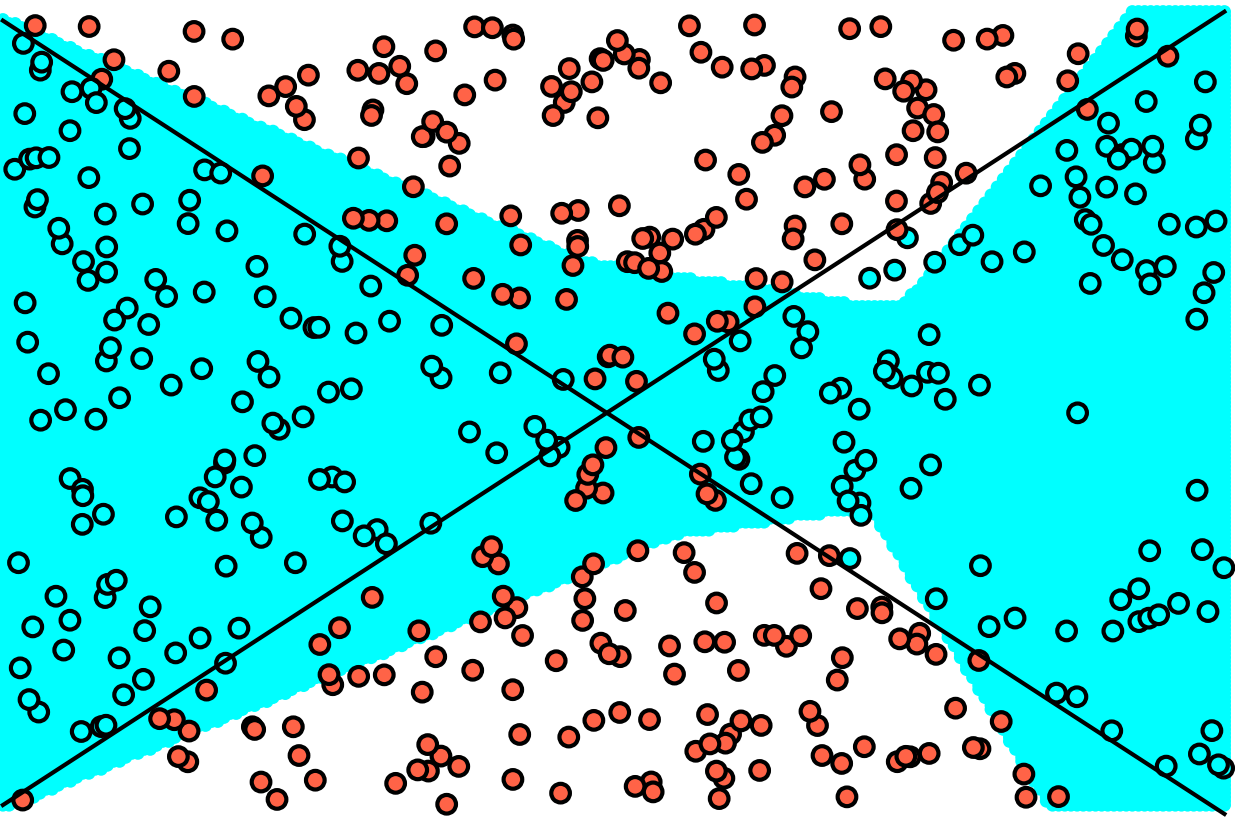}
	&\includegraphics[width=0.26\linewidth]{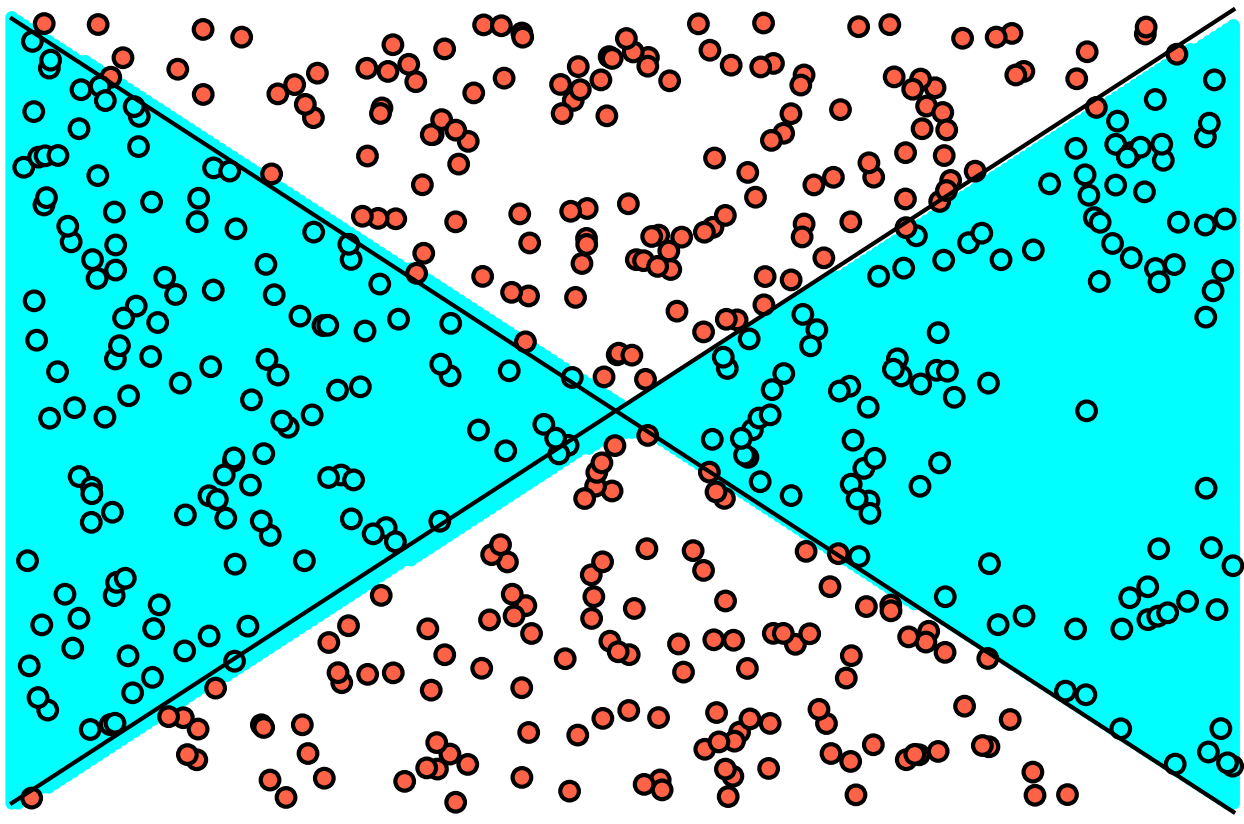}\\
	Initialization & After 150 Iterations & After 3k Iterations
	\end{tabular}
	\end{center}
	\vspace{-4mm}
	\caption{Deep Neural Rule: Evolution of the trainable support of a single rule with time}\label{fig:evolutionDeepNeural}
	\vspace{-1mm}
\end{figure}

\subsection{Real Data Analysis}
\noindent{\bf Datasets}. In order to compare the performance of the proposed algorithm with state of the art approaches, we perform an empirical evaluation on simulated and multiple real datasets, which ensures a wide variety of different targets in terms of their dependence on the input features. 

For simulation, we use a highly non-linear and multivariate artificial dataset, MADELON, featured in the NIPS 2003 feature selection challenge \cite{guyon2004result}. 
It is a  generalization of the classic XOR dataset to five dimensions. 
Each vertex of a five dimensional hypercube contains a cluster of data points randomly  labeled as $+1$ or $-1$. 
The five dimensions constitute $5$ informative features and  $15$ linear combinations of those features were added to form a set of 20 redundant but informative features.  
Additionally, a number of distractor features with no predictive power were added and the order of the features was randomized.

For benchmarking on real datasets, we will use Penn Machine Learning Benchmark (PMLB) \cite{Olson2017PMLB} which includes  datasets from a wide range of sources such as UCI ML repository \cite{Dua:2017}, Kaggle, KEEL \cite{Alcala-FdezFLDG11} and the meta-learning benchmark \cite{Reif12}. 
Since we are limiting our focus to binary classification tasks, we only consider datasets having two classes. 
Additionally, we removed all the datasets with fewer than $2000$ training examples. This leaves us with a total of $19$ datasets for our investigation.

\noindent{\bf Statistical Tests}. We use a statistical framework for hypothesis testing to investigate whether Neural Rule Ensembles (NRE) is significantly better or not compared to state of the art classifiers, namely Random Forests (RF), Gradient boosted trees (GB) and Artificial Neural Networks (ANN). 
A hypothesis test is a decision between two complementary hypotheses, the null hypothesis $H_0$ and the alternate hypothesis $H_1$. 
We are trying to reject the null hypothesis, which states that there is no difference in the classification performance of algorithms, that is,  both of them perform equally well. 
We use the following statistical tests designed to compare two classifiers on multiple data sets.

\noindent{\bf Wilcoxon Signed-Rank Test}.
For the Wilcoxon signed-rank test \cite{wilcoxon}, the results are sorted by the magnitude of absolute difference in the performance scores of the two classifiers. 
This is followed by assigning ranks from the lowest to the highest absolute difference. In case of ties, average ranks are assigned. 
Finally, a test statistic is formed based on the ranks of the positive and negative differences.

  \begin{table}[t]
	\vspace{-2mm}
	\centering
	\begin{tabular}{lcc|cc}

		{} &     GB &    NRE &  difference &  rank \\ 	\hline

		wilt             &  18.60 &  \textbf{10.40 }&        8.20 &  19.0  \\
		madelon          &  14.50 &  \textbf{10.30} &        4.20 &  18.0  \\
		adult            &  \textbf{12.91 }&  14.22 &       -1.31 &  17.0  \\
		phoneme          &   9.25 &  \textbf{ 8.14} &        1.11 &  16.0  \\
		dis              &\textbf{   0.71} &   1.77 &       -1.06 &  15.0  \\
		titanic          &  27.49 & \textbf{ 26.89} &        0.60 &  14.0  \\
		churn            &   \textbf{3.60} &   4.13 &       -0.53 &  13.0  \\
		banana           &   9.31 &   \textbf{8.93} &        0.38 &  12.0  \\
		ring             &   \textbf{3.15} &   3.51 &       -0.36 &  11.0  \\
		spambase         &   \textbf{4.34} &   4.63 &       -0.29 &  10.0  \\
		kr-vs-kp         &   0.42 &   \textbf{0.20} &        0.22 &   9.0  \\
		chess            &   \textbf{0.21} &   0.42 &       -0.21 &   7.5  \\
		coil2000         &   6.04 &  \textbf{ 5.83} &        0.21 &   7.5  \\
		twonorm          &   2.34 &   \textbf{2.25} &        0.09 &   6.0  \\
		clean2           &   0.00 &   0.00 &        0.00 &   3.0  \\
		hypothyroid      &   1.47 &   1.47 &        0.00 &   3.0  \\
		agaricus-lepiota &   0.00 &   0.00 &        0.00 &   3.0  \\
		magic            &  11.67 &  11.67 &        0.00 &   3.0  \\
		
		mushroom         &   0.00 &   0.00 &        0.00 &   3.0  \\
		
		\hline
		
		 wins & 6 &\textbf{8} & &   \\
		
		 ties & 5 & 5 & &   \\	
	\end{tabular}
\vspace{+2mm}
	\caption{Comparison of the test error performance of NRE with Gradient Boosted trees (GB) on binary classification tasks}\label{tab:GBNRE}
\vspace{-5mm}
\end{table}
Let $d_i$ be the difference between the performance scores of  two classifiers on the $i^{th}$ data set.  
Let $R^+$ be the sum of ranks for the data sets on which NRE outperforms the other classifier, and $R^{-}$ the sum of ranks on data sets where NRE gets defeated. 
Ranks corresponding to zero difference are split evenly between $R^+$ and $R^{-}$; if there is an odd number of them, one is ignored.  The test statistic, $T$ is given by

\vspace{-3mm}
 \begin{equation}
 T = \min(R^+,R^{-}) \\
\vspace{-3mm}
 \end{equation}

 where
 
\vspace{-3mm}
 \begin{equation}
 R^+ = \sum_{d_i>0} \rank(d_i) + \frac{1}{2} \sum_{d_i=0} \rank(d_i) \hspace{5mm}
\vspace{-1mm}
 \end{equation}
\vspace{-1mm}
 \begin{equation}
 R^{-} = \sum_{d_i<0} \rank(d_i) + \frac{1}{2} \sum_{d_i=0} \rank(d_i) \\\vspace{2mm}
\vspace{-2mm}
 \end{equation}

 For a two-tailed test with $\alpha =0.05$ significance level, the critical value of the test statistic corresponding to $n=19$ data sets is $46$. 
 In other words, if $T$ is less than or equal to $46$, NRE can be considered significantly better than the other classifier with $p < 0.05$ and we can reject the null-hypothesis in favor of alternate one.  

\noindent{\bf Sign Test: Wins, Losses \& Ties Counts}.
The sign test \cite{Salzberg1997,Sheskin:2007:HPN:1529939} is much weaker than the Wilcoxon signed-rank test and will not reject the null-hypothesis unless one algorithm almost always outperforms the other. 
In the sign test, we compare the generalization performance of classifiers by counting the number of data sets on which a classifier outperforms others.

Under the assumption that null-hypothesis is correct, that is, both classifiers perform equally well, one would expect each one of them to win on approximately $N/2$ out of $N$ data sets. 
This tell us that the number of wins is distributed according to the binomial distribution. 

For $19$ datasets, the critical number of wins needed to reject the null-hypothesis for a two-tailed sign test at $\alpha = 0.05$ significance is $14$. 
This implies that NRE can be considered significantly better than the other classifier with $p < 0.05$ if it is the overall winner on $14$ out of $19$ datasets. 
Since null-hypothesis is true for ties, instead of throwing them, we distribute them evenly between the two classifiers. 
And, we ignore one of the ties if there is an odd number of them.   
\begin{table}[t]
	\vspace{-2mm}
	\centering
	\begin{tabular}{lcc|cc}		
		 
		{} &     RF &    NRE &  difference &  rank   \\
		
		\hline
		
		madelon          &  26.40 &  \textbf{10.30} &       16.10 &  19.0   \\
		wilt             &  21.60 &  \textbf{10.40} &       11.20 &  18.0   \\
		coil2000         &   7.06 &   \textbf{5.83} &        1.23 &  17.0   \\
		phoneme          &   9.00 &   \textbf{8.14} &        0.86 &  16.0   \\
		banana           &   9.56 &   \textbf{8.93} &        0.63 &  15.0   \\
		titanic          &  27.49 &  \textbf{26.89 }&        0.60 &  14.0   \\
		spambase         &   4.92 &   \textbf{4.63} &        0.29 &  13.0   \\
		twonorm          &   2.52 &   \textbf{2.25} &        0.27 &  12.0   \\
		adult            &  14.47 &  \textbf{14.22} &        0.25 &  11.0   \\
		kr-vs-kp         &   0.42 &   \textbf{0.20} &        0.22 &  10.0   \\
		magic            &  11.88 &  \textbf{11.67 }&        0.21 &   9.0   \\
		hypothyroid      &   1.68 &   \textbf{1.47} &        0.21 &   8.0   \\
		chess            &   0.62 &  \textbf{ 0.42} &        0.20 &   7.0   \\
		ring             &   \textbf{3.33 }&   3.51 &       -0.18 &   6.0   \\
		agaricus-lepiota &   0.00 &   0.00 &        0.00 &   3.0   \\
		mushroom         &   0.00 &   0.00 &        0.00 &   3.0   \\
		dis              &   1.77 &   1.77 &        0.00 &   3.0   \\
		clean2           &   0.00 &   0.00 &        0.00 &   3.0   \\
		
		churn            &   4.13 &   4.13 &        0.00 &   3.0   \\
		
		\hline						
		wins & 1&\textbf{13 }& &   \\		
		ties & 5 & 5 & &   \\		
	\end{tabular}
\vspace{+2mm}
	\caption{Comparison of the test error performance of NRE with Random Forests (RF) on binary classification tasks}\label{tab:RFNRE}
	\vspace{-5mm}
\end{table}

\noindent{\bf Test Error Evaluation}. In this section, we compare NRE with GradientBoost (GB), Random Forest (RF) and Artificial Neural Networks (ANN) on $19$ datasets. The test errors for the datasets without a test set are obtained using five-fold cross-validation. \\
For each classifier, the operating settings and the tuned hyperparameters are the following:
\begin{itemize}

	\item \textbf{{Random Forests:} }The number of trees used in the forest are tuned from the set $k \in \{32,64,128,256,512\}$.
	\item \textbf{{Gradient Boosted Trees:} }We use $100$ boosting iterations with the maximum tree depth $d$ selected from the range  $d\in \{2,4,6,8,10\}$.
	\item \textbf{Artificial Neural Networks:} Fully connected networks with a single hidden layer (since NRE contains one hidden layer) and rectified linear (ReLU) activation.  The number of hidden units $h \in \{64,128,256,512,1024\}$ is selected for optimal performance. 
	\item {\textbf{Neural Rule Ensembles:}} Maximum depth of the tree used for initializing the network is searched over the set $d\in \{2,4,6,8,10\}$.
\end{itemize}

The hyperparameters for the methods being evaluated have been obtained by internal five-fold cross-validation on the training set. We use the scikit-learn implementation for evaluating the existing algorithms.  
\begin{table}[t]
	\vspace{-2mm}
	\centering

	\begin{tabular}{rrr|rr}		
		 			 
		{} &    ANN &    NRE &  difference &  rank  \\				\hline
		
		madelon          &  45.50 &  \textbf{10.30} &       35.20 &  19.0  \\
		phoneme          &  14.18 &   \textbf{8.14} &        6.04 &  18.0  \\
		wilt             &  14.20 &  \textbf{10.40} &        3.80 &  17.0  \\
		churn            &   6.27 &   \textbf{4.13} &        2.14 &  16.0  \\
		coil2000         &   7.46 &  \textbf{ 5.83} &        1.63 &  15.0  \\
		spambase         &   \textbf{3.47 }&   4.63 &       -1.16 &  14.0  \\
		ring             &   \textbf{2.52} &   3.51 &       -0.99 &  13.0  \\
		magic            &  12.44 &  \textbf{11.67} &        0.77 &  12.0  \\
		adult            &  14.79 &  \textbf{14.22} &        0.57 &  11.0  \\
		hypothyroid      &   1.89 &   \textbf{1.47} &        0.42 &   9.5  \\
		kr-vs-kp         &   0.62 &   \textbf{0.20} &        0.42 &   9.5  \\
		banana           &   9.31 &  \textbf{ 8.93} &        0.38 &   8.0  \\
		twonorm          &   2.43 &   \textbf{2.25} &        0.18 &   7.0  \\
		dis              &   1.94 &   \textbf{1.77} &        0.17 &   6.0  \\
		agaricus-lepiota &   0.00 &   0.00 &        0.00 &   3.0  \\
		mushroom         &   0.00 &   0.00 &        0.00 &   3.0  \\
		clean2           &   0.00 &   0.00 &        0.00 &   3.0  \\
		chess            &   0.42 &   0.42 &        0.00 &   3.0  \\
			
		titanic          &  26.89 &  26.89 &        0.00 &   3.0  \\
		
		\hline
				
		wins & 2&\textbf{12 }& &   \\
			
		ties & 5 & 5 & &   \\
		
	\end{tabular}
\vspace{+2mm}
	\caption{Comparison of the test error performance of NRE with Artificial Neural Networks (ANN) on binary classification tasks}\label{tab:ANNNRE}
	\vspace{-5mm}
\end{table}

\noindent{\bf NRE vs Gradient Boosted Trees}. From Table \ref{tab:GBNRE}, it can be seen that NRE wins on $8$ data sets, GB wins on $6$ data sets and there are $5$ ties. 
Ignoring one tie and splitting the remaining ones evenly, we find that NRE is better on $10$ out of $19$ datasets.  
Since the critical number of wins needed under sign test is 14, we fail to reject the null-hypothesis. 
Similarly, we fail to reject the null-hypothesis under the Wilcoxon signed-rank test because the test statistic  $T = \min(R^+,R^{-}) = \min(109,81) = 81 $ is greater than 46. 
This implies that we don't have enough statistical evidence to establish that NRE outperforms GB. 
However, we realize that NRE initialized from a single tree gives a tough competition to $100$ boosted trees and is a more compactly represented  model.

\noindent{\bf NRE vs Random Forest}. We find from Table \ref{tab:RFNRE} that NRE outperforms RF on almost all the data sets except for the ring data set and the 5 tied matches. 
Splitting the ties evenly, NRE is better on $15$ out of $19$ data sets which is greater than the critical number of wins needed, that is $14$, under the sign test. We can therefore reject the null hypothesis. 
For Wilcoxon-signed ranks test, the statistic $T = \min(R^+, R^{-}) = \min (176.5,13.5) = 13.5$ is less than the critical value $46$ which allows us to reject the null hypothesis as well.  
This implies that NRE is significantly better than Random Forest and given that it utilizes only one tree compared to the up to $500$ trees in RF, it is more compact too. 

\noindent{\bf NRE vs Artificial Neural Network}. It is evident from Table \ref{tab:ANNNRE} that NRE outperforms ANN on $12$ data sets, loses on $2$ sets and there are $5$ ties. 
NRE passes  the sign test since it is better on $14$ data sets (splitting the ties evenly) which matches the critical number of wins needed.  
Since, the test statistic for the Wilcoxon signed-rank test $T = \min(R^+, R^{-}) = \min (155.5,34.5) = 34.5$ is less than the critical value $46$, we reject the null-hypothesis in favor of alternate one. Both of the statistical tests agree that NRE is significantly better than the Artificial Neural Networks.

\begin{table}[t]
	\vspace{-2mm}
	\centering
	\begin{tabular}{lcccccc}

		Dataset &$N$ & $p$    & GB &    RF &  ANN &  NRE \\ 	\hline

		wilt       &4839 &6      &  18.60 & 21.60& 14.20& \textbf{10.40} \\
		madelon    &2600 &500      &  14.50 & 26.40& 45.50 & \textbf{10.30}   \\
		phoneme    &5404     &6 &  9.25 &9.00 & 14.18 &\textbf{8.14} \\
		kr-vs-kp   &3197  &37   &   0.42 &0.42  &0.62 &\textbf{0.20} \\
		coil2000   &9822 &86     &   6.04 &7.06 &7.46 &\textbf{5.83}\\
        banana     &5300 &3      &   9.31 & 9.56 &9.31  &\textbf{8.93}  \\
        twonorm    &7400    &21      &   2.34 &2.52 &2.43 &\textbf{2.25} \\        
        adult      &48842    &15  &   \textbf{12.91 }& 14.47& 14.79 &  14.22 \\
		dis        &3772 &30     &   \textbf{0.71} &1.77 &1.94 &  1.77 \\
		churn      &5000 &21     &   \textbf{3.60} &4.13& 6.27 &4.13\\
		ring       &7400 &21      &   \textbf{3.15} &3.33 &2,52 &3.51 \\
		spambase   &4601 &58      &   \textbf{4.34} & 4.92 &3.47 & 4.63  \\
		chess      &3196 &37      &   \textbf{0.21} &0.62 &0.42 &0.42 \\
		titanic     &2201 &4     &  27.49 & 27.49 &\textbf{26.89} &  \textbf{26.89}  \\
		hypothyroid  &3163 &26    &  \textbf{1.47} &1.68 &1.89  &\textbf{1.47} \\
		magic         &19020 &11   &  \textbf{11.67} &11.88 &12.44  &\textbf{11.67}  \\
		mushroom       &8124 &23  & \textbf{0.00} &\textbf{0.00}  &\textbf{0.00} &\textbf{0.00} \\
		clean2         &6598 &169  &  \textbf{0.00} &\textbf{0.00}  &\textbf{0.00} &\textbf{0.00}  \\
		agaricus-lepiota &8145 &23 &   \textbf{0.00} &\textbf{0.00}  &\textbf{0.00} &\textbf{0.00} \\
		\hline
		wins & & & 11 &3 &4 &13   \\
	\end{tabular}
	\caption{Comparison of the test error performance of NRE with Gradient Boosted trees (GB), Random Forests (RF) and Artificial Neural Networks (ANN) on binary classification tasks}\label{tab:comparison}
	\vspace{-5mm}
\end{table}

\noindent{\bf Overall comparison}. 
 In Table \ref{tab:summary} is shown a summary of the Wilcoxon rank T statistics and the number of NRE wins vs the other methods, with their significance in bold. 
 In Table \ref{tab:comparison}, are shown all the classification test errors for all the methods in a single table. Also shown are the number of observations $N$ and the number of features $p$ of each dataset. 

\begin{table}[t]
\vspace{+3mm}
	\centering
	\begin{tabular}{lccc}		
		\hline
&NRE vs GB &NRE vs RF &NRE vs ANN  \\
		\hline
	Wilcoxon T Statistic	&81 &{\bf 13.5} &{\bf 34.5}\\
	Number of NRE wins  	&10 &{\bf 15} &{\bf 14}\\
		\hline
	\end{tabular}
\vspace{+2mm}
	\caption{Summary results comparing the NRE with GB, RF and ANN. }\label{tab:summary}
\vspace{-6mm}
\end{table}

\section{Conclusion}

In this work, we presented a novel method called Neural Rule Ensembles (NRE) for encoding into a neural network and refining the feature interactions captured by a decision tree. 
This was achieved by defining a neural transformation of a tree-induced rule using ReLU units and the min pooling operation. 
Such a mapping addresses the initialization related concerns of fully connected neural networks as well as the feature selection problem, and enables learning of compact  representations compared to conventional tree-based approaches.

 Empirical evaluations on $19$ binary classification datasets from the Penn Machine Learning Benchmark (PMLB) \cite{Olson2017PMLB} were performed to compare the generalization performance of Neural Rule Ensembles (NRE) with state of the art approaches such as Random Forests (RF), Gradient Boosted Trees (GB) and Artificial Neural Networks (ANN). 
We used two statistical tests, the Wilcoxon signed-rank test and the sign test, to evaluate the statistical significance of these results.  
Both of these statistical tests found NRE to be significantly better than Random Forests and the Artificial Neural Networks  with $p < 0.05$.  
When NRE was compared to Gradient Boosted Trees, we could not find enough statistical evidence to reject the null hypothesis stating that both of them perform equally well.  
However, NRE only utilizes one tree, so it obtains a more compact and interpretable representation.

\bibliographystyle{IEEEtran}
\bibliography{IEEEabrv, IEEEexample}

\begin{thebibliography}{10}
\providecommand{\url}[1]{#1}
\csname url@samestyle\endcsname
\providecommand{\newblock}{\relax}
\providecommand{\bibinfo}[2]{#2}
\providecommand{\BIBentrySTDinterwordspacing}{\spaceskip=0pt\relax}
\providecommand{\BIBentryALTinterwordstretchfactor}{4}
\providecommand{\BIBentryALTinterwordspacing}{\spaceskip=\fontdimen2\font plus
\BIBentryALTinterwordstretchfactor\fontdimen3\font minus
  \fontdimen4\font\relax}
\providecommand{\BIBforeignlanguage}[2]{{%
\expandafter\ifx\csname l@#1\endcsname\relax
\typeout{** WARNING: IEEEtran.bst: No hyphenation pattern has been}%
\typeout{** loaded for the language `#1'. Using the pattern for}%
\typeout{** the default language instead.}%
\else
\language=\csname l@#1\endcsname
\fi
#2}}
\providecommand{\BIBdecl}{\relax}
\BIBdecl

\bibitem{friedman2001greedy}
J.~H. Friedman, ``Greedy function approximation: a gradient boosting machine,''
  \emph{Annals of Statistics}, pp. 1189--1232, 2001.

\bibitem{breiman2001random}
L.~Breiman, ``Random forests,'' \emph{Machine Learning}, vol.~45, no.~1, pp.
  5--32, 2001.

\bibitem{nalenz2016horseshoe}
M.~Nalenz, ``Horseshoe rulefit: Learning rule ensembles via bayesian
  regularization,'' 2016.

\bibitem{blaszczynski2016multi}
J.~B{\l}aszczy{\'n}ski, B.~Prusak, and R.~S{\l}owi{\'n}ski, ``Multi-objective
  search for comprehensible rule ensembles,'' in \emph{International Joint
  Conference on Rough Sets}.\hskip 1em plus 0.5em minus 0.4em\relax Springer,
  2016, pp. 503--513.

\bibitem{de2017best}
K.~W. De~Bock, ``The best of two worlds: Balancing model strength and
  comprehensibility in business failure prediction using spline-rule
  ensembles,'' \emph{Expert Systems with Applications}, vol.~90, pp. 23--39,
  2017.

\bibitem{article}
M.~Nalenz and M.~Villani, ``Tree ensembles with rule structured horseshoe
  regularization,'' \emph{The Annals of Applied Statistics}, vol.~12, 02 2017.

\bibitem{deng2019interpreting}
H.~Deng, ``Interpreting tree ensembles with intrees,'' \emph{International
  Journal of Data Science and Analytics}, vol.~7, no.~4, pp. 277--287, 2019.

\bibitem{Quinlan:1993:CPM:152181}
J.~R. Quinlan, \emph{C4.5: Programs for Machine Learning}.\hskip 1em plus 0.5em
  minus 0.4em\relax San Francisco, CA, USA: Morgan Kaufmann Publishers Inc.,
  1993.

\bibitem{friedman2008predictive}
J.~H. Friedman and B.~E. Popescu, ``Predictive learning via rule ensembles,''
  \emph{The Annals of Applied Statistics}, pp. 916--954, 2008.

\bibitem{softrules}
D.~Akdemir, N.~Heslot, and J.-L. Jannink, ``Soft rule ensembles for supervised
  learning,'' pp. 78--83, 01 2013.

\bibitem{Dembczynski2008}
\BIBentryALTinterwordspacing
K.~Dembczy\'{n}ski, W.~Kotlowski, and R.~Slowi\'{n}ski, ``Maximum likelihood
  rule ensembles,'' in \emph{Proceedings of the 25th International Conference
  on Machine Learning}, ser. ICML '08.\hskip 1em plus 0.5em minus 0.4em\relax
  New York, NY, USA: ACM, 2008, pp. 224--231. [Online]. Available:
  \url{http://doi.acm.org/10.1145/1390156.1390185}
\BIBentrySTDinterwordspacing

\bibitem{sethi1991decision}
\BIBentryALTinterwordspacing
I.~K. Sethi, ``Decision tree performance enhancement using an artificial neural
  network implementation1 1this work was supported in part by nsf grant
  iri-9002087,'' in \emph{Artificial Neural Networks and Statistical Pattern
  Recognition}, ser. Machine Intelligence and Pattern Recognition, I.~K. SETHI
  and A.~K. JAIN, Eds.\hskip 1em plus 0.5em minus 0.4em\relax North-Holland,
  1991, vol.~11, pp. 71 -- 88. [Online]. Available:
  \url{http://www.sciencedirect.com/science/article/pii/B9780444887405500104}
\BIBentrySTDinterwordspacing

\bibitem{neuralforests}
J.~Welbl, ``Casting random forests as artificial neural networks (and profiting
  from it),'' in \emph{Pattern Recognition}, X.~Jiang, J.~Hornegger, and
  R.~Koch, Eds.\hskip 1em plus 0.5em minus 0.4em\relax Cham: Springer
  International Publishing, 2014, pp. 765--771.

\bibitem{deepDecision}
P.~Kontschieder, M.~Fiterau, A.~Criminisi, and S.~R. Bulò, ``Deep neural
  decision forests,'' in \emph{2015 IEEE International Conference on Computer
  Vision (ICCV)}, Dec 2015, pp. 1467--1475.

\bibitem{Ioannou2016DecisionFC}
Y.~Ioannou, D.~P. Robertson, D.~Zikic, P.~Kontschieder, J.~Shotton, M.~Brown,
  and A.~Criminisi, ``Decision forests, convolutional networks and the models
  in-between,'' \emph{CoRR}, vol. abs/1603.01250, 2016.

\bibitem{Olson2017PMLB}
\BIBentryALTinterwordspacing
R.~S. Olson, W.~La~Cava, P.~Orzechowski, R.~J. Urbanowicz, and J.~H. Moore,
  ``Pmlb: a large benchmark suite for machine learning evaluation and
  comparison,'' \emph{BioData Mining}, vol.~10, no.~1, p.~36, Dec 2017.
  [Online]. Available: \url{https://doi.org/10.1186/s13040-017-0154-4}
\BIBentrySTDinterwordspacing

\bibitem{hornik1991approximation}
K.~Hornik, ``Approximation capabilities of multilayer feedforward networks,''
  \emph{Neural Networks}, vol.~4, no.~2, pp. 251--257, 1991.

\bibitem{backprop}
\BIBentryALTinterwordspacing
D.~E. Rumelhart, G.~E. Hinton, and R.~J. Williams, ``Neurocomputing:
  Foundations of research,'' J.~A. Anderson and E.~Rosenfeld, Eds.\hskip 1em
  plus 0.5em minus 0.4em\relax Cambridge, MA, USA: MIT Press, 1988, ch.
  Learning Representations by Back-propagating Errors, pp. 696--699. [Online].
  Available: \url{http://dl.acm.org/citation.cfm?id=65669.104451}
\BIBentrySTDinterwordspacing

\bibitem{adamOptimization}
\BIBentryALTinterwordspacing
D.~P. Kingma and J.~Ba, ``Adam: A method for stochastic optimization.''
  \emph{CoRR}, vol. abs/1412.6980, 2014. [Online]. Available:
  \url{http://dblp.uni-trier.de/db/journals/corr/corr1412.html\#KingmaB14}
\BIBentrySTDinterwordspacing

\bibitem{guyon2004result}
I.~Guyon, S.~Gunn, A.~Ben-Hur, and G.~Dror, ``Result analysis of the nips 2003
  feature selection challenge,'' in \emph{NIPS}, 2004, pp. 545--552.

\bibitem{Dua:2017}
\BIBentryALTinterwordspacing
D.~Dheeru and E.~Karra~Taniskidou, ``{UCI} machine learning repository,'' 2017.
  [Online]. Available: \url{http://archive.ics.uci.edu/ml}
\BIBentrySTDinterwordspacing

\bibitem{Alcala-FdezFLDG11}
\BIBentryALTinterwordspacing
J.~Alcalá-Fdez, A.~Fernández, J.~Luengo, J.~Derrac, and S.~García, ``Keel
  data-mining software tool: Data set repository, integration of algorithms and
  experimental analysis framework.'' \emph{Multiple-Valued Logic and Soft
  Computing}, vol.~17, no. 2-3, pp. 255--287, 2011. [Online]. Available:
  \url{http://dblp.uni-trier.de/db/journals/mvl/mvl17.html\#Alcala-FdezFLDG11}
\BIBentrySTDinterwordspacing

\bibitem{Reif12}
\BIBentryALTinterwordspacing
M.~Reif, ``A comprehensive dataset for evaluating approaches of various
  meta-learning tasks.'' in \emph{ICPRAM (1)}, P.~L. Carmona, J.~S. Sánchez,
  and A.~L.~N. Fred, Eds.\hskip 1em plus 0.5em minus 0.4em\relax SciTePress,
  2012, pp. 273--276. [Online]. Available:
  \url{http://dblp.uni-trier.de/db/conf/icpram/icpram2012-1.html\#Reif12}
\BIBentrySTDinterwordspacing

\bibitem{wilcoxon}
\BIBentryALTinterwordspacing
F.~Wilcoxon, ``Individual comparisons by ranking methods,'' \emph{Biometrics
  Bulletin}, vol.~1, no.~6, pp. 80--83, 1945. [Online]. Available:
  \url{http://www.jstor.org/stable/3001968}
\BIBentrySTDinterwordspacing

\bibitem{Salzberg1997}
\BIBentryALTinterwordspacing
S.~L. Salzberg, ``On comparing classifiers: Pitfalls to avoid and a recommended
  approach,'' \emph{Data Mining and Knowledge Discovery}, vol.~1, no.~3, pp.
  317--328, Sep 1997. [Online]. Available:
  \url{https://doi.org/10.1023/A:1009752403260}
\BIBentrySTDinterwordspacing

\bibitem{Sheskin:2007:HPN:1529939}
D.~J. Sheskin, \emph{Handbook of Parametric and Nonparametric Statistical
  Procedures}, 4th~ed.\hskip 1em plus 0.5em minus 0.4em\relax Chapman \&
  Hall/CRC, 2007.

\end{thebibliography}

\end{document}